\documentclass[lettersize,journal]{IEEEtran}
\usepackage[caption=false,font=normalsize,labelfont=sf,textfont=sf]{subfig}

\usepackage{multicol}
\usepackage[bookmarks=true]{hyperref}

\usepackage{bm}
\usepackage{amssymb}
\usepackage{dsfont}
\usepackage{amsmath}

\usepackage{algorithm}
\usepackage{algorithmicx}
\usepackage{algpseudocode}

\usepackage{amsthm}
\usepackage{xcolor}
\usepackage{graphicx}
\usepackage{url}
\usepackage{svg}
\usepackage{nicefrac}
\usepackage{balance}
\usepackage{mathtools}
\usepackage{booktabs}
\usepackage{siunitx}
\usepackage{multirow} 
\usepackage{pgfplots}
\usepackage{pgfplotstable}
\usepackage{tikz}
\pgfplotsset{compat=newest}
\usepgfplotslibrary{groupplots}
\usepgfplotslibrary{units}
\usepgfplotslibrary{statistics}

\definecolor{customred}{RGB}{220,33,77}
\definecolor{lightred}{RGB}{255,150,150}
\definecolor{customblue}{RGB}{0,100,222}
\definecolor{lightblue}{RGB}{150,200,255}
\newcommand{\rebuttaladd}[1]{\textcolor{black}{#1}}

\newtheoremstyle{exampstyle}
  {3pt} 
  {3pt} 
  {\itshape} 
  {} 
  {\bfseries} 
  {.} 
  {.5em} 
  {} 

\theoremstyle{exampstyle} 

\newtheorem{lemma}{Lemma}
\newtheorem{theorem}{Theorem}
\newtheorem{remark}{Remark}

\newtheorem{example*}{Example*}

\theoremstyle{plain}

\DeclareMathOperator*{\argmin}{arg\,min}

\newcommand{\RNum}[1]{\uppercase\expandafter{\romannumeral #1\relax}}
\newcommand{\ofx}{\left(\bm{x}\right)}

\newcommand{\xb}{\bm{x}}
\newcommand{\ub}{\bm{u}}

\algtext*{EndFor}
\algtext*{EndIf}

\begin{document}

\title{
Parameter-Robust MPPI for Safe Online Learning\\
of Unknown Parameters}

\author{Matti Vahs\textsuperscript{1}, Jaeyoun Choi\textsuperscript{2}, Niklas Schmid\textsuperscript{3}, Jana Tumova\textsuperscript{1} and Chuchu Fan\textsuperscript{2}
	\thanks{\textsuperscript{1}Matti Vahs and Jana Tumova are with the Division of Robotics, Perception and Learning, KTH Royal Institute of Technology, Stockholm, Sweden and also affiliated with Digital Futures. Mail addresses: {\{\tt\small vahs, tumova\}}
		{\tt\small @kth.se}. Their work was partially supported by the Wallenberg AI, Autonomous
		Systems and Software Program (WASP) funded by the Knut and Alice
		Wallenberg Foundation. }
    \thanks{\textsuperscript{2}Jaeyoun Choi and Chuchu Fan are with the Reliable Autonomous Systems Lab, Massachusetts Institute of Technology, Cambridge, MA, USA. Mail addresses: {\{\tt\small cjy0051, chuchu\}}
		{\tt\small @mit.edu}. The authors were partly funded by the Ministry of Trade, Industry and Energy (MOTIE), Korea, through the Global Industrial Technology Cooperation Center (GITCC) Program, supervised by the Korea Institute for Advancement of Technology (KIAT) (Task No. P24680172).}
    \thanks{\textsuperscript{3}Niklas Schmid is with the Automatic Control Laboratory, Swiss Federal Institute of Technology in Z\"urich, Switzerland. Mail address: {\tt\small nikschmid}
		{\tt\small @ethz.ch}. The author was supported by the Swiss National Science Foundation under NCCR Automation under grant 51NF40\_225155 and an additional Mobility Grant.}%
}



\maketitle

\begin{abstract}
Robots deployed in dynamic environments must remain safe even when key physical parameters are uncertain or change over time. We propose Parameter-Robust  Model Predictive Path Integral~(PRMPPI) control, a framework that integrates online parameter learning with probabilistic safety constraints. PRMPPI maintains a particle-based belief over parameters via Stein Variational Gradient Descent, evaluates safety constraints using Conformal Prediction, and optimizes both a nominal performance-driven and a safety-focused backup trajectory in parallel. This yields a controller that is cautious at first, improves performance as parameters are learned, and ensures safety throughout. Simulation and hardware experiments demonstrate higher success rates, lower tracking error, and more accurate parameter estimates than baselines.
\end{abstract}


\section{Introduction}
Robotic systems are often deployed in dynamic environments where reliable control must be maintained despite imperfect knowledge of the system’s dynamics. Sometimes the key parameters of a robot model -- such as friction, inertial properties or payload -- are not known precisely in advance and may even change over time. 
\emph{Domain randomization} is a popular technique, where a controller is robustified against a predefined distribution of models ~\cite{peng2018sim}.
While effective, this strategy is typically static: the uncertainty set is chosen a-priori and does not leverage the stream of measurements available during operation. In safety-critical tasks, this mismatch leads to either conservative performance (if the randomization is broad) or brittle behavior (if it is narrow). Ideally, to reduce significant engineering work in identifying a system's parameters, one would obtain a controller that is cautious initially but leverages measurements to refine parameters, improves performance over time, and maintains safety throughout.

A natural alternative to static domain randomization is to \emph{learn parameters online} via Bayesian estimation and adapt control to the evolving belief. Classical approaches trace back to system identification \cite{ljung1983theory}, with modern variants leveraging differentiable simulators~\cite{heiden2022probabilistic} or black-box simulation for likelihood-free updates \cite{barcelos2020disco}. In many such pipelines, however, estimation is treated as a module decoupled from control~\cite{rohr2023credible}, or the controller is designed for a single point estimate (e.g., maximum likelihood) of the parameters--an instance of certainty equivalence \cite{9636325}. Model Predictive Control (MPC) offers a way to couple learning with decision-making by updating models within a receding horizon. Gaussian Process (GP) residual modeling is a popular approach to adapt dynamics online and to expose epistemic uncertainty to the controller~\cite{deisenroth2013gaussian}.
Particle-based estimators have also been used to track parameter beliefs and to inform control \cite{abraham2020model, Barcelos-RSS-21}, but only optimize \emph{expected} cost and lack \emph{probabilistic} safety guarantees over the full prediction horizon. For safety-critical robotics, a challenge is not just adapting to parameter uncertainty, but ensuring that entire trajectories remain safe with high probability.
\begin{figure}
    \centering
    \includegraphics[width=0.47\textwidth]{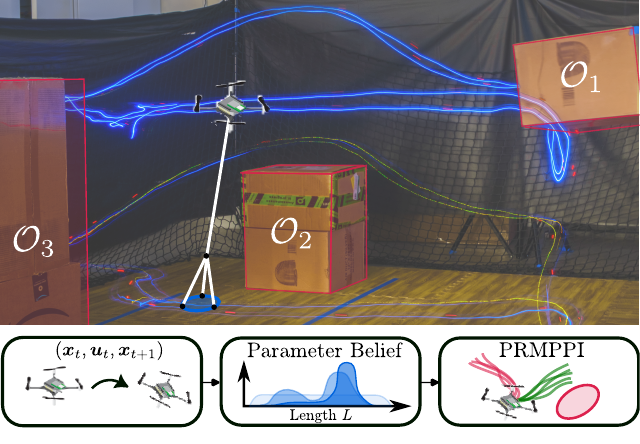}
    \vspace{-0.4cm}
    \caption{Long exposure shot of our hardware experiments of a Crazyflie 2.1 brushless quadrotor with a suspended payload of unknown length. The quadrotor has to track a square trajectory while avoiding collisions and has to improve its belief about its physical parameters from online observations.}
    \label{fig:firstpage}
    \vspace{-0.5cm}
\end{figure}

In our work, we aim to bridge this gap by combining particle-based online parameter learning with explicit probabilistic safety enforcement over full trajectories, rather than optimizing expected cost alone.
Three challenges arise in making this practical. \emph{First}, posterior distributions over physical parameters in nonlinear systems are generally non-Gaussian, which limits filters that assume a simple form. \emph{Second}, safety constraints are expressed as chance constraint over a receding horizon: guaranteeing that a robot will remain safe in the future requires reasoning about predicted trajectories
and obtaining closed-form, differentiable expressions for such constraints is difficult. \emph{Third}, embedding these
safety requirements into a tractable control strategy remains challenging, as constraints must be enforced repeatedly in real time under an evolving belief. Fig.~\ref{fig:firstpage} shows our proposed method on a quadrotor. By refining its belief about its payload length, it is eventually able to take shortcuts to safely improve performance.

Specifically, we propose parameter-robust model predictive path integral (PRMPPI) control that
\begin{enumerate}
    \item enables online learning of uncertain parameters \emph{safely} \rebuttaladd{by using a non-Gaussian particle belief},
    \item improves performance during the execution of a task by directly using the learned distribution in the controller,
    \item improves classic MPPI approaches through parallel optimization of \emph{robust} backup trajectories.
\end{enumerate}
An overview of our approach is shown in Fig.~\ref{fig:firstpage}. We maintain a particle-based belief over unknown parameters using Stein Variational Gradient Descent (SVGD). We then assess safety of candidate trajectories by sampling parameter hypotheses and computing conformal-prediction–based robustness scores. Finally, we embed these components into an MPPI controller that optimizes a nominal trajectory while maintaining a parallel robust backup trajectory. This enables safe online parameter refinement since the parameter uncertainty is explicitly propagated into trajectory evaluation and safety is enforced throughout the prediction horizon.
\subsection{Related Work} 
If parameter uncertainty is known in advance and can be described by a bounded set, robust controllers can be synthesized through approaches such as robust MPC~\cite{mayne2005robust}, Hamilton–Jacobi reachability analysis~\cite{mitchell2005time}, learning robust Control Barrier Functions~\cite{dawson2022learning}, or online rollouts~\cite{knoedler2025safety}. These methods typically consider uncertain parameters as adversarial noise. While this perspective yields rigorous guarantees, it can be unnecessarily conservative in settings where parameters can be progressively identified during operation. 

In online settings, a prominent approach is to use Gaussian Process (GP) regression to model residual dynamics or unmodeled disturbances and embed this model within MPC~\cite{hewing2020learning}, or within reinforcement learning to learn policies directly~\cite{deisenroth2011pilco}. The variance of the GP prediction provides a natural quantification of epistemic uncertainty, enabling chance-constrained optimal control formulations~\cite{hewing2020learning}. To make such formulations tractable, different approximations for uncertainty propagation have been proposed: linearization via Taylor expansion~\cite{hewing2019cautious} or the unscented transform~\cite{ostafew2016robust}. These methods work when uncertainty is small and unimodal, but become less reliable for larger uncertainties and highly nonlinear dynamics.

Reinforcement learning often employs domain randomization to bridge the sim-to-real gap, training policies that are robust to a wide range of simulator parameters~\cite{peng2018sim}. While effective for transfer, static randomization again ignores the fact that system parameters can be progressively inferred during deployment. Recent extensions propose \emph{adaptive domain randomization}~\cite{mehta2020active, possas2020online}, where both the parameter distribution and the policy are updated online using real-world observations. These methods improve performance over time but lack explicit safety enforcement during adaptation.

To handle non-Gaussian and potentially multimodal parameter distributions, several works adopt nonparametric particle-based methods \cite{abraham2020model, barcelos2020disco, wang2021adaptive, Barcelos-RSS-21}. In particular,~\cite{abraham2020model, barcelos2020disco} sample parameter hypotheses from a belief distribution and evaluate them within a Model Predictive Path Integral (MPPI) control framework, while~\cite{Barcelos-RSS-21} applies Stein Variational Gradient Descent~\cite{liu2016stein} to jointly update the parameter belief and the trajectory distribution. These approaches demonstrate the feasibility of combining online parameter inference with sampling-based control. However, they focus exclusively on minimizing expected cost and do not incorporate safety constraints, making them unsuitable for safety-critical robotic applications.

\section{Preliminaries and Problem Formulation}
\subsection{Stein Variational Gradient Descent}
\label{sec:SVGD}
Stein Variational Gradient Descent (SVGD) is a particle-based variational inference method that transforms an initial distribution $q\ofx$ into a target distribution $p_t\ofx$ by iteratively transporting a set of particles to reduce the KL divergence $\mathrm{KL}(q\|p_t)$ \cite{liu2016stein}.

Specifically, the candidate distribution $q\ofx$ is represented nonparametrically by a set of particles 
$\{\xb^{(i)}\}_{i=1}^N$ 
and is updated according to
\begin{align*}
    \xb_{k+1}^{(i)} &= \xb_{k}^{(i)} + \epsilon \,\phi^*(\xb_{k}^{(i)}),
\end{align*}
where the transport direction $\phi^*$ is chosen to maximally decrease the KL divergence,
\begin{align*}
    \phi^*(\xb) 
    &= \underset{\phi \in \Phi}{\mathrm{argmax}}
    \Big(-\frac{\mathrm{d}}{\mathrm{d}\epsilon} 
    \mathrm{KL}(q(\xb+\epsilon\phi)\|p_t) \big|_{\epsilon=0}\Big).
\end{align*}
Here, $\Phi$ is a Reproducing Kernel Hilbert Space (RKHS) induced by a kernel function $k(\xb,\xb')$. It is shown in \cite{liu2016stein} that this parameterization yields a closed-form solution
\begin{align*}
    \phi^*(\cdot) = \mathbb{E}_{\xb\sim q}\big[\nabla_{\xb}\log p_t\ofx \, k(\xb,\cdot) + \nabla_{\xb} k(\xb,\cdot)\big],
\end{align*}
which can be approximated empirically as
\begin{align}
    \phi^*(\cdot)=\frac{1}{N}\sum_{i=1}^N\big[\nabla\log p_t(\xb^{(i)})\,k(\xb^{(i)},\cdot)
    +\nabla k(\xb^{(i)},\cdot)\big] \label{eq:SVGDUpdate}
\end{align}
by using particles. In effect, SVGD transports particles toward high-probability regions of $p_t$ while maintaining diversity through the repulsive kernel term. Further, although distributions are approximated using a finite number of particles, it is shown in \cite{shi2023finite} that the transported distribution converges to the true target in the mean field limit, i.e. as $N\rightarrow \infty$.

\subsection{Conformal Prediction}
\label{sec:CP}
Conformal Prediction (CP) is a lightweight statistical tool for uncertainty quantification that can enable practical safety guarantees for autonomous systems \cite{shafer2008tutorial}.
\rebuttaladd{In our setting, CP provides a way to quantify the uncertainty in violating safety constraints based on samples of the parameter distribution which are propagated through the robot dynamics to calculate the safety probability.}
\begin{lemma}
    Let $Z, Z^{(1)},\dots, Z^{(k)}$ be k + 1 independent and identically distributed (i.i.d.) real-valued random variables. Let $Z^{(1)},\dots, Z^{(k)}$ be sorted in non-decreasing order and define $Z^{k+1} := \infty$. For $\delta \in (0, 1)$, it holds that 
    \begin{align*}
        \mathrm{Pr}\left[Z \leq \bar{Z}\right] \geq 1- \delta
    \end{align*}
    where $\bar{Z} = Z^{(r)}$ with $ r = \lceil(k+1)(1-\delta)\rceil$
    and where $\lceil\cdot\rceil$ is the ceiling function.
    \label{lemma:CP}
\end{lemma}
The minimum number of required samples to provide a probabilistic guarantee is given by 
$k \geq \frac{1 - \delta}{\delta}$. The variable $Z$ is usually referred to as the \emph{non-conformity score}.

\subsection{Model Predictive Path Integral Control}
\label{sec:MPPI}
MPPI is a control method to solve stochastic Optimal Control Problems (OCPs) for discrete-time dynamical systems
\begin{align*}
    \xb_{k+1} = \bm{F}(\xb_k, \bm{v}_k), \hspace{0.5cm} \bm{v}_k \sim \mathcal{N}(\ub_k, \bm{\Sigma}).
\end{align*}
MPPI samples $M$ random control input sequences $\bm{v}_{0:K-1}^{(1:M)}$ of length $K-1$ and forward simulates the system dynamics given the current state $\xb_0$ to obtain $X^{(m)} = [\xb_0, \bm{F}(\xb_0, \bm{v}_0^{(m)}), \dots, \bm{F}(\xb_{K-1}, \bm{v}_{K-1}^{(m)})]$. Then, given the state rollouts and a cost function $J(X)$ to be minimized, each rollout is weighted by an importance sampling weight 
\begin{align}
    w^{(m)} = \frac{1}{\eta} \exp \left(- \frac{1}{\beta} \left(J(X^{(m)}) - \rho)\right)\right),\label{eq:ImportanceWeight}
\end{align}
where $\eta$ is a normalization constant ensuring $\sum_{m=1}^M w^{(m)} = 1$, $\rho = \min_m J(X^{(m)})$ is subtracted for numerical stability and $\beta$ is the \emph{inverse temperature} which serves as a tuning parameter for the sharpness of the control distribution. Finally, an approximate optimal control sequence can be obtained as
\begin{align*}
    \ub_{0:K-1}^* = \sum_{m=0}^M w^{(m)} \bm{v}_{0:K-1}^{(m)},
\end{align*}
which is a weighted average of sampled control trajectories and is applied in a receding horizon fashion. For a detailed discussion and theoretical properties, we refer to \cite{williams2017model}

\subsection{Problem Setting}
Consider a robot that is modeled by a discrete-time system
\begin{align}
    \xb_{k+1} = \bm{f}(\xb_k, \ub_k, \theta)\label{eq:Dynamics}
\end{align}
where $\xb \in \mathcal{X} \subset \mathbb{R}^{n_x}$ is the state, $\ub \in \mathcal{U} \subset \mathbb{R}^{n_u}$ is the control input and $\theta \in \Theta \subset \mathbb{R}^{n_{\theta}}$ are unknown parameters of the system such as a robot's mass or its inertial properties. The parameters are only known up to a prior distribution on the parameters $p(\theta_0)$, e.g. a uniform distribution over masses. 

During operation the robot observes its state and has to complete a safety-critical task which is expressed as finite horizon chance constrained optimal control problem
\begin{equation}
        \begin{aligned}
         \min_{\bm{u}_0, \dots, \ub_{N-1} }~  &J = \mathbb{E}\left\{\sum_{k=0}^{N-1} \ell(\xb_k,\ub_k) + L (\xb_N)\right\}\\
        \text{s.t. } \quad &\xb_{k+1} = \bm{f}(\xb_k, \ub_k, \theta),\hspace{1.4cm}\forall k = 0, \dots, N-1\\
        \quad &\ub_k \in \mathcal{U},\hspace{3.35cm}\forall k = 0, \dots, N-1\\
        &\mathrm{Pr}\left[\cap_{k=0}^N \xb_k \in \mathcal{C}\right] \geq 1 - \delta\\
        &\xb_0 = \xb(t),\hspace{0.3cm} \theta \sim p(\theta(t))
        \end{aligned}
    \label{eq:OCP}
\end{equation}
where $\ell: \mathcal{X} \times \mathcal{U} \mapsto \mathbb{R}$ is the stage cost, $L: \mathcal{X}\mapsto\mathbb{R}$ is the terminal cost and 
    $\mathcal{C} = \left\{\xb \in \mathcal{X} \mid h\ofx \geq 0\right\}\subseteq \mathcal{X}$
is a safe set within which we want the state to remain with a desired probability threshold $1-\delta$. \rebuttaladd{The safe set is defined as the superlevel set of some continuous function $h:\mathcal{X}\mapsto \mathbb{R}$ which could be, e.g., the collision-free space.}  
 
We approximate the solution to Eq.~\eqref{eq:OCP} by 1) representing the posterior parameter distribution using an SVGD-based particle filter, 2) quantifying the particle distribution based safety probability using CP and 3) incorporating this uncertainty into a MPPI framework that solves the OCP.

\section{Parameter-Robust MPPI}
\subsection{Online Parameter Estimation}
Accurate estimation of uncertain system parameters is essential for our proposed control framework. A principled way to infer these parameters is to use Bayes' rule, which updates a prior belief $p(\theta\mid \xb_{1:t},\ub_{1:t-1})$ about the parameters $\theta$ using new state observations $\xb_{t+1}$ and applied control inputs $\ub_{t}$
\begin{align}
    p(\theta\!\mid\!\xb_{1:t+1}, \ub_{1:t}) 
    = \frac{p(\xb_{t+1}\!\mid\!\theta, \xb_{t}, \ub_{t})\,
            p(\theta\!\mid\!\xb_{1:t}, \ub_{1:t})}
           {p(\xb_{t+1}\!\mid\!\xb_{1:t}, \ub_{1:t})}.\label{eq:Bayes}
\end{align}
One can think of this update as shifting the probability mass to parameter values that best explain the observed state transition.
While conceptually straightforward, this posterior is generally intractable for nonlinear systems since the denominator requires to solve an integral over the entire parameter space $\Theta$.

A common approach is to approximate the posterior using a set of particles $\{\theta^{(i)}\}_{i=1}^N$ and to update them using importance sampling~\cite{thrun2005probabilistic}. While simple, such Sequential Importance Resampling (SIR) filters tend to suffer from weight degeneracy, where after several updates only a few particles carry most probability mass, causing mode collapse and poor exploration in multi-modal or high-dimensional posteriors.
We therefore adopt SVGD, which transports an unweighted set of particles toward high-probability regions while preserving diversity, avoiding the aforementioned mode collapse problem.

In order to apply SVGD as described in Sec.~\ref{sec:SVGD}, we note that the update in Eq.~\eqref{eq:SVGDUpdate} only depends on the gradient of the log-likelihood of the target density. Since the intractable denominator in Eq.~\eqref{eq:Bayes} only serves as a normalization constant, we can define the target density
\begin{align}
    p_t(\theta) 
    = p(\xb_{t+1}\mid \theta, \xb_{t}, \ub_{t})\;
      p(\theta\mid\xb_{1:t}, \ub_{1:t})\label{eq:Target}
\end{align}
which is proportional to $p(\theta\!\mid\!\xb_{1:t+1}, \ub_{1:t}) $. As a consequence, $\nabla_{\theta}\mathrm{log}~p_t = \nabla_{\theta}\mathrm{log~}p(\theta\mid \xb_{1:t+1}, \ub_{1:t})$.
The first term of the target in Eq.~\eqref{eq:Target} is given by the system dynamics
\begin{align*}
    \xb_{t+1} = \bm{f}(\xb_{t}, \ub_{t}, \theta) + \bm{\xi},
    \hspace{0.3cm} \bm{\xi}\sim\mathcal{N}(\bm{0},\bm{\Sigma}_\xi),
\end{align*}
where we assume zero-mean Gaussian observation noise with known covariance $\bm{\Sigma}_\xi$. Since we do not directly assume noisy state measurements in our problem setting, we can either estimate the covariance from real-world data or use it as a tuning parameter to account for a model mismatch.

To apply SVGD to the target, we further require the prior $p(\theta\mid\xb_{1:t},\ub_{1:t})$ to be differentiable with respect to $\theta$. Since our belief is represented by particles which is not directly differentiable, we perform a kernel density estimation (KDE)
\begin{align}
    p(\theta\mid \xb_{1:t},\ub_{1:t})
    \approx \frac{1}{N}\sum_{i=1}^N K_{\rebuttaladd{\sigma}}(\theta,\theta^{(i)})\label{eq:KDE}
\end{align}
where $K_{\rebuttaladd{\sigma}}: \Theta \times\Theta\mapsto \mathbb{R}$ is a positive definite kernel with bandwidth $\rebuttaladd{\sigma}>0$. This ensures a continuous, differentiable belief that enables gradient-based updates. Further it allows sampling new parameter candidates which is important for the uncertainty quantification discussed in Sec.~\ref{sec:UncQuant}. Crucially, the quality of the estimated density depends on the correctness of the bandwidth $\rebuttaladd{\sigma}$. 
\rebuttaladd{In practice, we use a radial basis function (RBF) kernel for $K_{\sigma}$ and select its bandwidth using Silverman's rule of thumb~\cite{silverman2018density}, which we found to provide stable density estimates for the parameter particles.}

Combining the transition likelihood and the prior, we can compute the gradient of the log unnormalized posterior as
\begin{align*}
    \nabla_{\theta}\log p_t(\theta) 
    =&\;\bm{\Sigma}_{\xi}^{-1}(\xb_{t+1} - \bm{f}(\xb_{t},\ub_{t},\theta))
       \nabla_{\theta}\bm{f} \\
    &+ \nabla_{\theta} \log \Big(\frac{1}{N}\sum_{i=1}^N K_{\rebuttaladd{\sigma}}(\theta,\theta^{(i)})\Big)
\end{align*}
which can be evaluated efficiently in closed form and used within the SVGD update rule to transport the parameter particles toward high-likelihood regions. 

\subsection{Receding Horizon Controller}
In this section, we discuss how to use the parameter belief $p(\theta\mid \xb_{1:t}, \ub_{1:t})$ in motion planning to guarantee probabilistic constraint satisfaction over a receding horizon. We calculate the probability of satisfaction along a nominal path, which we then incorporate into a sampling-based MPPI controller.
\subsubsection{Uncertainty Quantification in Safety Violations}
\label{sec:UncQuant}
Given the current state $\xb_t$ and a control input trajectory over a receding horizon $N$, i.e. $\ub_0, \dots, \ub_{N-1}$, we can propagate the state through the dynamics model to obtain a state trajectory $\xb_{t:t+N}$. However, since we have uncertainty in the model's parameter $\theta$, the evolution of the state becomes a random process.

Following Sec.~\ref{sec:CP}, we define the non-conformity score
\begin{align}
    \rho = - \underset{k=0,\dots,N}{\mathrm{min}} \ h(\xb_k)\label{eq:nonconformity}
\end{align}
which can be interpreted as the \rebuttaladd{negative} distance to the unsafe set along a rolled out trajectory. The non-conformity score is a mapping $\rho : \Theta \mapsto \mathbb{R}$ from the space of parameters to a scalar value, given a control sequence $\ub_{0:N-1}$. Since the parameters are uncertain, the non-conformity score is a random variable as well. However, we can draw samples from the distribution over non-conformity scores by first generating samples from $p(\theta\mid\xb_{1:t}, \ub_{1:t})$ and then propagating the dynamics under each parameter hypothesis $\theta^{(i)}$, the current state $\xb_t$ and the control input sequence. 

\begin{theorem}
    Let $\theta^{(1)},\dots, \theta^{(k)}$ be $k$ i.i.d. samples drawn from $p(\theta\mid \xb_{1:t}, \ub_{1:t})$. Let $\rho^{(1)},\dots, \rho^{(k)}$ be the resulting non-conformity scores sorted in increasing order and define $\rho^{(k+1)} := \infty$. If it holds that $\rho^{(r)} \leq 0$ with $r = \lceil(k+1)(1-\delta)\rceil$, then
    \begin{align*}
        \mathrm{Pr}\left[\cap_{k=0}^N \xb_k \in \mathcal{C}\right] \geq 1- \delta
    \end{align*}
    holds for a given probability threshold $\delta \in (0, 1)$.
    \label{thm:CP}
\end{theorem}
\begin{proof}
    First, we note that the following two statements
    \begin{align*}
        \mathrm{Pr}\left[\cap_{k=0}^N \xb_k \in \mathcal{C}\right] \geq 1- \delta\\
        \Leftrightarrow\mathrm{Pr}\left[\underset{k=0,\dots,N}{\mathrm{min}} \ h(\xb_k) \geq 0\right] \geq 1- \delta
    \end{align*}
    are equivalent. Thus, by defining a non-conformity score according to Eq.~\eqref{eq:nonconformity} and sampling i.i.d. parameters, we get i.i.d. samples from the distribution over non-conformity scores. According to Lemma~\ref{lemma:CP}, we get
    \begin{align*}
        \mathrm{Pr}\left[- \underset{k=0,\dots,N}{\mathrm{min}} \ h(\xb_k) \leq \rho^{(r)}\right] \geq 1-\delta
    \end{align*}
    \rebuttaladd{with $\rho^{(r)}\leq 0$} which concludes that the true trajectory realization is safe with desired probability.
\end{proof}
\begin{remark}
    \label{remark1}
    The proposed probabilistic safety guarantee only holds if $\theta$ is sampled from the true distribution over parameters. However, there is a mismatch between our belief of $p(\theta\mid\xb_{1:t}, \ub_{1:t})$ and the true posterior since we only use a finite number of particles in practice. \rebuttaladd{This mismatch means that the safety guarantee may become slightly conservative or occasionally violated when the particle approximation poorly captures the true posterior. Nonetheless, because SVGD maintains diverse particles and updates them online, we generally observe that the estimated belief remains sufficiently accurate for the CP-based test to provide reliable behavior in real-world settings.}
    Such mismatch could be addressed by considering recently proposed error bounds for finite particle systems in SVGD~\cite{shi2023finite} and combine them with extended versions of CP such as robust CP which we leave for future work.
\end{remark}

\subsubsection{Control Design}
\label{sec:Control}
We propose two extensions of MPPI to approximate the OCP's solution. Specifically, we propose extensions to handle constraint satisfaction which is not straight forward in standard MPPI. 

First, when rolling out the dynamics for a control sequence $V^{(j)} := \ub_{1:K-1} + \bm{\eta}_{1:K-1}^{(j)}$ perturbed by Gaussian noise $\bm{\eta}$, we sample $P=\lceil\nicefrac{(1-\delta)}{\delta}\rceil$ i.i.d. parameters from the current belief $p(\theta\mid\xb_{1:t}, \ub_{1:t})$ and rollout the dynamics model in Eq.~\eqref{eq:Dynamics} for each parameter hypothesis. 
This will result in $P$ state trajectory samples for which we can calculate the non-conformity scores in Eq.~\eqref{eq:nonconformity} to obtain the \emph{robustness} $R := - \rho^{(r)}$ from Theorem~\ref{thm:CP}. Note that if $R >0$, the resulting trajectory is safe with desired probability $1-\delta$. Thus, the joint chance constraint defined in Eq.~\eqref{eq:OCP} can be evaluated in a simple sampling-based manner.

MPPI, as introduced in Sec.~\ref{sec:MPPI}, does not natively handle constraint satisfaction. Thus, one common way of enforcing constraints is by penalizing them in the cost function 
\begin{align}
    \vphantom{}^NJ = \mathbb{E}_{p}\left\{J\left(\xb_{1:N}\right)\right\} + \rebuttaladd{W} \cdot \mathds{1}\left\{R\left(\xb_{1:N}^{(1:P)}\right)  < 0\right\}\label{eq:CostNominal}
\end{align}
where $\rebuttaladd{W}\gg 0$ is a large penalization constant \rebuttaladd{and P denotes the number of parameter samples}. However, this simply transforms the safety constraint to a soft constraint which can still result in safety violations as MPPI can diverge when all sampled control sequences become unsafe. This can happen, e.g., if the environment changes too fast or if the simulated dynamics during the MPPI rollouts do not capture the noise in the actual dynamics.

To remedy this, we propose a parallel optimization structure where, in addition to a nominal optimization of the objective function in Eq.~\eqref{eq:CostNominal}, we also optimize a robust backup trajectory which is purely dedicated to maximizing the safety probability. This can be expressed as a robust cost function
\begin{align}
    \vphantom{}^R J = - R\left(\xb_{1:N}^{(1:P)}\right)\label{eq:robust_cost}
\end{align}
that we seek to minimize in parallel. Thus, at each control step we are optimizing two trajectories simultaneously, a \emph{nominal} (N) and a \emph{robust} (R) trajectory which we refer to as $\vphantom{}^N U$ and $\vphantom{}^R U$, respectively. The main benefit of the robust trajectory is that we can bootstrap the nominal optimization from both trajectories which prevents the aforementioned problem of all rollouts suddenly becoming infeasible. An example of such scenario is showcased in Fig.~\ref{fig:TrajVis}.

\newlength{\textfloatsepsave} \setlength{\textfloatsepsave}{\textfloatsep} 

\setlength{\textfloatsep}{0pt}
\begin{algorithm}[t]
\caption{Parameter-Robust MPPI}
\label{algo:PRMPPI}
\begin{algorithmic}[1]
\State \textbf{Initialize:} $\vphantom{}^N U = \vphantom{}^R U = \mathbf{0}_{n_u \times K-1}$
\State $\vphantom{}^N U, \vphantom{}^R U \gets \textsc{TimeShift}(\vphantom{}^N U, \vphantom{}^R U)$
\State $\theta^{(1:P)} \gets$ \textsc{SampleParameters()}
\ForAll{\rebuttaladd{samples }$m = 0$ to $M$ \textbf{in parallel}}
    \State $\bm{\eta}_{1:K-1}^{(m)} \gets$ \textsc{SampleNoiseSequence()}
    \State $\vphantom{}^N V^{(m)} =\vphantom{}^NU + \bm{\eta}_{1:K-1}^{(m)}$; $\vphantom{}^R V^{(m)} =\vphantom{}^RU + \bm{\eta}_{1:K-1}^{(m)}$ 
    \State $(\vphantom{}^NJ_1^{(m)}, \vphantom{}^RJ_1^{(m)}) \gets \textsc{Dyn}(\xb_t, \vphantom{}^N V^{(m)}, \theta^{(1:P)})$
    \State $(\vphantom{}^NJ_2^{(m)}, \vphantom{}^RJ_2^{(m)}) \gets \textsc{Dyn}(\xb_t, \vphantom{}^R V^{(m)}, \theta^{(1:P)})$
\EndFor
\State $\vphantom{}^Nw_1, \vphantom{}^Nw_2 \gets \textsc{ImportanceSampling}(\vphantom{}^NJ_1, \vphantom{}^NJ_2)$ \hfill $\triangleright$ \eqref{eq:ImportanceWeight}
\State $\vphantom{}^N U_1, \vphantom{}^N U_2 \gets \sum_{m}\vphantom{}^Nw_1^{(m)}\vphantom{}^N V^{(m)}, \sum_{m}\vphantom{}^Nw_2^{(m)}\vphantom{}^R V^{(m)}$
\State $\vphantom{}^N U = \argmin_{i\in\{1, 2\}}~ \textsc{cost}(\textsc{Dyn}(\xb_t, \vphantom{}^N U_i, \theta^{(1:P)}))$
\State $\vphantom{}^R w \gets \textsc{ImportanceSampling}( \vphantom{}^R J_2)$ \hfill $\triangleright$ \eqref{eq:ImportanceWeight}
\State $\vphantom{}^R U = \sum_{m}\vphantom{}^Rw^{(m)}\vphantom{}^R V^{(m)}$
\If{$\textsc{Robustness}(\textsc{Dyn}(x_0, \vphantom{}^N U, \theta^{(1:P)})) > 0$}
\State \textbf{return:} $\vphantom{}^N\ub_0$
\Else
\State $\vphantom{}^NU \gets \vphantom{}^RU\hspace{0.3cm}$ \textbf{return:} $\vphantom{}^R\ub_0$
\EndIf
\end{algorithmic}
\end{algorithm}
\setlength{\textfloatsep}{\textfloatsepsave}

We summarize the resulting algorithm in Alg.~\ref{algo:PRMPPI}. First, we initialize two trajectories to be optimized. At every control step, the previous loop solution is shifted using the $\textsc{TimeShift}(\cdot)$ operator and $P$ i.i.d. parameters are sampled from the current KDE of the particle belief in Eq.~\eqref{eq:KDE}. Then, we generate $2\cdot M$ disturbed control sequences, bootstrapped both from the nominal and the robust trajectory (line 6). For each input sequence, we rollout the dynamics under all parameter hypotheses and calculate the resulting expected costs and robustnesses which is summarized in the $\textsc{Dyn}(\cdot)$ function. \rebuttaladd{We want to highlight that although the two rollouts of nominal and robust trajectories are shown sequentially in lines 7-8, they can be performed in parallel as they are completely independent.} In lines $9-11$, we optimize both trajectory distributions with respect to the nominal cost in Eq.~\eqref{eq:CostNominal} and pick the minimal one. Similarly, the robust trajectory is optimized in lines $12-13$ with respect to Eq.~\eqref{eq:robust_cost}. Finally, to ensure constraint satisfaction, we rollout the nominal control sequence (line 14) and apply the robust control input if safety is violated.

\section{Experiments}
We evaluate our method in simulation as well as hardware experiments in which model inaccuracies in terms of uncertain parameters occur. Specifically, we
\begin{enumerate}
    \item show improved performance over existing baselines, both in terms of performance and safety in Sec.~\ref{sec:Baselines},
    \item validate that parallel optimization yields safer behavior in fully and partially observable settings in Sec.~\ref{sec:AblationRobust}, 
    \item showcase real-world performance improvement through online learning on a quadcopter with cable-suspended payload with complex safety requirements in Sec.~\ref{sec:Hardware}.
\end{enumerate}
\vspace{-0.2cm}
\begin{figure}[t]
    \centering
    \includegraphics[width=1\linewidth]{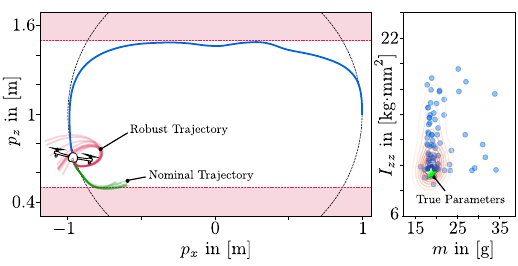}
    \vspace{-0.7cm}
    \caption{Illustration of the simulated 2D Quadrotor environment. \textbf{Left:} The circular path to be tracked is shown in black while the traversed path is colored in blue. The nominal trajectory is entering the unsafe region which is shaded in red. \textbf{Right:} A snapshot of the parameter belief in the depicted time step. Particles are denoted by blue circles and the contours of the KDE are shown in red.}
    \vspace{-0.2cm}
    \label{fig:TrajVis}
\end{figure}
\begin{table*}[ht]
\centering
\renewcommand{\arraystretch}{1.1}
\setlength{\tabcolsep}{10pt}
\caption{Comparison Across Environments in \texttt{safe-control-gym}. All results are summarized as mean and standard deviation, i.e. $\mu \pm \sigma$.}
\vspace{-0.3cm}
\begin{tabular}{l|ccc|ccc}
\toprule
\textbf{Method} & \multicolumn{3}{c|}{\textbf{Cartpole}} & \multicolumn{3}{c}{\textbf{Quadrotor}} \\
 & RMSE $\downarrow$ & Success Rate (SR) $\uparrow$ & PA (\%) $\uparrow$ & RMSE $\downarrow$ & Success Rate (SR) $\uparrow$ & PA (\%) $\uparrow$ \\
\midrule
Oracle  &$0.56 \pm 0.08$        &$100/100$        &$-$  & $0.15 \pm 0.02$       &  $100/100$      &$-$              \\
\hline
Nominal MPPI   &$1.37 \pm 1.36$        &$64/100$        &$-$ & $0.16 \pm 0.1$       &  $10/100$      &$-$               \\
Robust MPPI    & $5.57 \pm 0.24$       &  $100/100$      & $-$ &    $0.41 \pm 0.14$    &  $74/100$      & $-$        \\
GPMPC        & $0.73 \pm 0.33$       &  $100/100$      &  $68.1 \pm 19.9$    &   $0.38 \pm 0.2$     &   $78/100$     &  $90.4 \pm 17$     \\
\rebuttaladd{PRMPPI + UKF}        & \rebuttaladd{$0.6 \pm 0.07$}       & \rebuttaladd{$100/100$}       &  \rebuttaladd{$86.8 \pm 8.4$}   &     \rebuttaladd{$0.24 \pm 0.04$}   &  \rebuttaladd{$97/100$}      &   \rebuttaladd{$96.6 \pm 3.3$}      \\
Ours        & $0.59 \pm 0.07$       & $100/100$       &  $90.69 \pm 6.3$   &     $0.17 \pm 0.03$   &  $100/100$      &   $99.5 \pm 0.73$      \\
\bottomrule
\end{tabular}
\vspace{-0.3cm}
\label{tab:method_comparison}
\end{table*}

\subsection{Simulation Setup}
We evaluate our method in the \texttt{safe-control-gym} benchmark environment \cite{yuan2022safe}, which provides nonlinear control tasks with safety constraints. Specifically, we use a constrained quadrotor tracking task and the classic cartpole benchmark. In the cartpole setting, the objective is to drive the cart to the origin $p_c = 0$ while keeping the pole upright, with the safety requirement that neither the cart nor the pole enters the region $p_c < 0$. The cart and pole masses are randomized uniformly within $\pm 10\%$ of their nominal values $m_c = 1\si{kg}$ and $m_p = 0.1\si{kg}$.

In the more challenging quadrotor setting, we consider a circular tracking task with a constrained height as illustrated in Fig.~\ref{fig:TrajVis}. We uniformly randomize the quadrotor's mass and its mass moment of inertia in a range of $\pm 50 \%$ around the nominal values of $m=27\si{g}$ and $I_z = 1.4 \cdot 10^{-5} \si{kg \cdot m^2}$.

We run 100 random simulations in which each simulation consists of 3 repetitive runs. 
This provides enough time to learn the parameters online. The cost function is defined as a quadratic penalization to the reference. 

Further, we set the safety probability to $1-\delta = 0.9$ and use $500$ samples for each, the robust and nominal optimization. All computations are implemented in JAX and performed on a computer with an AMD Ryzen 7 9800X3D CPU, 64 GB of RAM, and an NVIDIA RTX 5090 GPU.

\subsubsection{Evaluation Metrics}
\rebuttaladd{We evaluate our approach with three metrics: root mean squared error (RMSE) measures tracking error against the reference, success rate (SR) reflects the fraction of runs satisfying safety constraints, and parameter accuracy (PA) compares the final parameter estimate to the ground truth, with 
100 \% indicating perfect accuracy.}

\subsubsection{Baselines}
\rebuttaladd{To approximate an upper bound on performance under perfect model knowledge, we include an oracle controller, implemented as a standard MPPI controller.} We also evaluate a nominal MPPI controller that assumes the nominal parameters throughout and does not adapt them online. To isolate the effect of our online parameter estimation module, we further include a robust MPPI baseline that does not update the parameters online, instead randomizing over the prior distribution. \rebuttaladd{Additionally, we add a baseline in which the parameter belief is obtained by an unscented Kalman Filter (UKF).} Lastly, we compare against GPMPC \cite{hewing2019cautious}, a nonlinear MPC approach that incrementally learns a residual dynamics model through Gaussian Process regression (GPR). Because GPR does not run at control frequency, we update the residual model after each of the three simulation runs. Further, we set the safety probability to \(90\%\) which is a pointwise-in-time guarantee. To calculate parameter accuracy, we solve a least-squares problem on the learned residual model using all collected data points to recover maximum likelihood estimates of the parameters. For fair comparison, all baselines use $10^4$ total rollouts as in our method. Only the oracle is allowed to use $10^5$ rollouts to find close-to-optimal solutions.

\begin{figure}
\begin{tikzpicture}

\begin{axis}[
    ybar,
    bar width=9pt,
    width=8.3cm,
    height=4cm,
    ymin=0,
    axis y line*=left,
    ylabel={\textcolor{customred}{RMSE $\downarrow$}},
    ylabel style={text=customred},
    xtick={1,2,3},
    xticklabels={Lap 1, Lap 2, Lap 3},
    xtick align=inside,
    ymajorgrids=false,
    xmin=0.6, xmax=3.4,
    legend style={
        at={(0.38,1.2)},
        anchor=north,
        legend columns=-1,
        /tikz/every even column/.append style={column sep=1cm},
        draw=none
    }
]

\addplot+[
    style={customred, fill=customred},
    fill opacity=0.5,
    error bars/.cd,
    y dir=both, y explicit,
    error bar style={thick, draw=customred},
] coordinates {
    (0.85,0.2055) +- (0,0.0332)
    (1.85,0.1714) +- (0,0.0443)
    (2.85,0.1713) +- (0,0.0800)
};

\addplot+[
    style={customred, fill=lightred},
    fill opacity=0.5,
    error bars/.cd,
    y dir=both, y explicit,
    error bar style={thick, draw=customred},
] coordinates {
    (0.84,0.6376) +- (0,0.1185)
    (1.84,0.2529) +- (0,0.0833)
    (2.84,0.2402) +- (0,0.0768)
};

\pgfplotsset{
    legend image code/.code={
        \draw[customred, fill=customred, fill opacity=0.5] (0cm,-0.1cm) rectangle (0.2cm,0.1cm);
        \draw[customblue, fill=customblue, fill opacity=0.5] (0.22cm,-0.1cm) rectangle (0.42cm,0.1cm);
    }
}
\addlegendimage{empty legend}
\addlegendentry{Ours}

\end{axis}

\begin{axis}[
    ybar,
    bar width=9pt,
    width=8.3cm,
    height=4cm,
    ymin=0,
    axis y line*=right,
    axis x line=none,
    ylabel={\textcolor{customblue}{PA $\uparrow$}},
    ylabel style={text=customblue},
    xtick=\empty,
    ymajorgrids=false,
    xmin=0.6, xmax=3.4,
    legend style={
        at={(0.62,1.2)},
        anchor=north,
        legend columns=-1,
        /tikz/every even column/.append style={column sep=1cm},
        draw=none
    }
]

\addplot+[
    style={customblue, fill=customblue},
    fill opacity=0.5,
    error bars/.cd,
    y dir=both, y explicit,
    error bar style={thick, draw=customblue},
] coordinates {
    (1.16,0.9814) +- (0,0.0425)
    (2.16,0.9948) +- (0,0.0080)
    (3.16,0.9939) +- (0,0.0103)
};

\addplot+[
    style={customblue, fill=lightblue},
    fill opacity=0.5,
    error bars/.cd,
    y dir=both, y explicit,
    error bar style={thick, draw=customblue},
] coordinates {
    (1.15,0.4677) +- (0,0.35)
    (2.15,0.7532) +- (0,0.32)
    (3.15,0.9027) +- (0,0.17)
};

\pgfplotsset{
    legend image code/.code={
        \draw[lightred, fill=lightred, fill opacity=0.5] (0cm,-0.1cm) rectangle (0.2cm,0.1cm);
        \draw[lightblue, fill=lightblue, fill opacity=0.5] (0.22cm,-0.1cm) rectangle (0.42cm,0.1cm);
    }
}
\addlegendimage{empty legend}
\addlegendentry{GPMPC}

\end{axis}

\end{tikzpicture}
\vspace{-1.0cm}
\caption{RMSE and PA over multiple laps in the simulated Quadrotor environment. Our method identifies parameters more quickly resulting in lower RMSEs.}
\label{fig:Barplot}
\vspace{-0.2cm}
\end{figure}
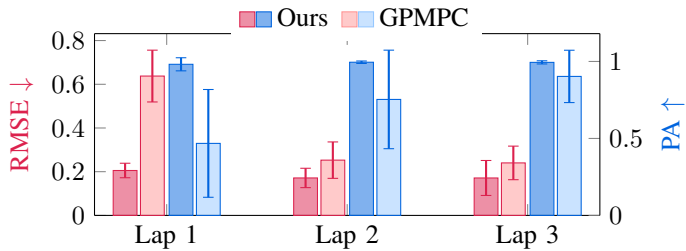

\subsubsection{Comparison to Baselines}
\label{sec:Baselines}
The quantitative results of the simulation study are summarized in Table~\ref{tab:method_comparison}. Our method outperforms all baselines in terms of RMSE, SR and PA in both environments. Only the nominal MPPI controller is not able to satisfy safety constraints in the cartpole setting which can be attributed to wrong parameters which results in incorrect rollouts. Robust MPPI is able to satisfy safety constraints in the cartpole environment at the cost of a high RMSE which indicates that this baseline performs overly conservative. In the quadrotor environment, the robust approach even leads to significant safety violations in terms of SR. This can be attributed to the robust approach not being able to stabilize the quadrotor under significant parameter randomizations of up to 50\% which makes the robust control problem challenging and eventually leads to the quadrotor crashing. \rebuttaladd{PRMPPI with a UKF instead of the SVGD belief performs competitively in terms of tracking but shows reduced parameter accuracy and slightly more safety violations, suggesting that a purely unimodal Gaussian belief may not fully capture the parameter uncertainty encountered in these tasks.} Lastly, GPMPC performs similarly in the cartpole environment but does not achieve good results in terms of SR in the quadrotor environment. 
We suspect that this behavior arises because GPMPC learns the system parameters only implicitly through GPR and therefore struggles to generalize beyond the data manifold on which it was trained. During the first lap, data is collected based on a uniform GP prior and subsequently used for training. However, in the second lap, the quadrotor follows a slightly different trajectory for which no data has been collected, causing the state to leave the previously explored data manifold and forcing the controller to rely largely on the uninformed prior again. This effect is reflected in Fig.~\ref{fig:Barplot}, where the parameter accuracy improves only gradually and converges to a reasonable estimate by the third round. In contrast, our method rapidly identifies the unknown parameters from the outset, which also leads to a lower RMSE.

\begin{figure*}[t]
    \centering
    \includegraphics[width=1\linewidth]{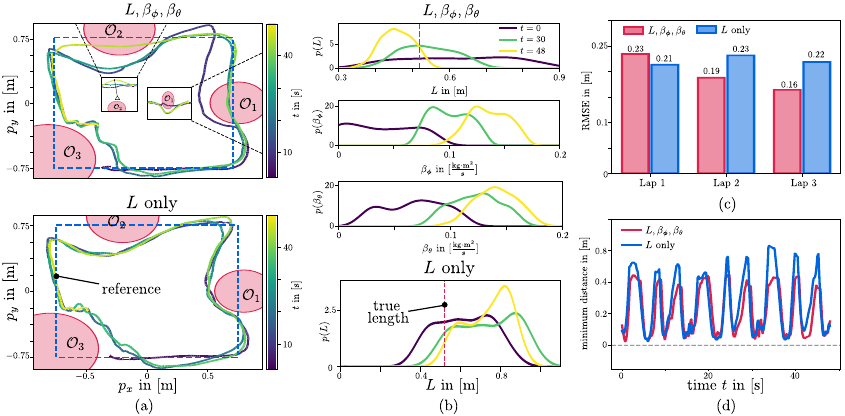}
    \vspace{-0.7cm}
    \caption{Experimental result for both parameter learning settings. (a) Illustration of the traversed $xy$ trajectories of the Crazyflie. Obstacles are shown in red and a colormap indicates the time scale. (b) The parameter beliefs at different time steps. (c) A barplot  indicating that learning damping parameters and length jointly can result in more accurate beliefs leading to lower \rebuttaladd{trajectory} RMSEs over time. (d) The minimum distance to obstacles over time.}
    \label{fig:Experiments}
    \vspace{-0.2cm}
\end{figure*}
Both PRMPPI and GPMPC perform worse in terms of PA in the cartpole environment. This could be because the pendulum mass becomes unobservable in the upright pole position. Thus, most of the state observations are uninformative which leads to an overall worse parameter belief.
\begin{table}[b]
\vspace{-0.5cm}
\centering
\caption{Influence of Robust Backup Trajectory.}
\vspace{-0.2cm}
\begin{tabular}{@{}l|cc|cc@{}}
\toprule
\multirow{2}{*}{\centering \textbf{Configuration}} 
& \multicolumn{2}{c|}{Fully Observable} 
& \multicolumn{2}{c}{Partially Observable} \\
 & \textbf{RMSE} & \textbf{SR} 
 & \textbf{RMSE} & \textbf{SR} \\
\midrule
w Robust Traj.    & $0.17 \pm 0.03$ & 100/100 & $0.18 \pm 0.04$ & 95/100 \\
wo Robust Traj.   & $0.18 \pm 0.03$ & 97/100  & $0.18 \pm 0.02$ & 86/100 \\
\bottomrule
\end{tabular}
\label{tab:ablation_robust}
\end{table}
\subsubsection{Ablation Robust Trajectory}
\label{sec:AblationRobust}
We further investigate the impact of the proposed parallel optimization structure for generating robust trajectories. To this end, we conduct an ablation study (Table~\ref{tab:ablation_robust}) in both the standard quadrotor environment and a partially observable variant. \rebuttaladd{In this setting, ''partial observability’’ means that height constraints are revealed only when the quadrotor is within 40 cm of an obstacle.}
The robust cost defined in Eq.~\eqref{eq:robust_cost} is applied under the conservative assumption that the current state is always subject to constraints within $\pm 40\si{cm}$. Our results show that, while this structure yields only minor safety gains in the fully observable case, it improves the success rate by more than $10\%$ in the partially observable environment. This demonstrates that robustly bootstrapping a cost-minimizing MPPI controller can substantially improve satisfaction of safety requirements.

\subsubsection{Ablation Probabilistic Guarantee}
\label{seq:AblationSamples}
Lastly, we evaluate the effect of the number of parameter samples $P$ on both performance and computation time in the quadrotor environment, see Table~\ref{tab:ablation_samples}. Interestingly, a safety probability of 
90\% yields the best performance, while lower probability thresholds achieve lower RMSE at the expense of increased safety violations. As the safety probability increases, the RMSE rises, which can be attributed to the controller adopting more conservative behavior. Notably, the 
95\% setting still exhibits a single collision, which could be due to the inherent stochasticity in parameter sampling. Importantly, we observe that all configurations can still operate at the control frequency of $50\si{Hz}$, even when using up to $10^5$ rollouts. The computation time includes the SVGD updates performed at each step.
\begin{table}[h!]
\vspace{-0.2cm}
\centering
\caption{Influence Of The Number of Samples.}
\vspace{-0.2cm}
\begin{tabular}{@{}l|cccc}
\toprule
\textbf{$1-\delta$} & \textbf{Samples $P$} & \textbf{RMSE} & \textbf{Comp. Time (ms)} & \textbf{SR} \\
\midrule
0.8 & 5 & $0.17 \pm 0.02$ & $9.22 \pm 0.32$ & 97/100 \\
0.9 & 10 & $0.17 \pm 0.03$ & $10.08 \pm 0.3$ & 100/100 \\
0.95 & 20 & $0.19 \pm 0.05$ & $10.65 \pm 0.33$ & 99/100 \\
0.99 & 100  & $0.19 \pm 0.04$ & $11.42 \pm 0.28$  & 100/100 \\
\bottomrule
\end{tabular}
\label{tab:ablation_samples}
\vspace{-0.5cm}
\end{table}

\subsection{Hardware Experiment}
\label{sec:Hardware}
We validate PRMPPI on a Crazyflie 2.1 brushless quadrotor with a cable-suspended payload operating in a cluttered indoor environment\footnote{Video available at \href{https://www.youtube.com/watch?v=xqmjUHMO3MQ}{https://www.youtube.com/watch?v=xqmjUHMO3MQ}} as shown in Fig.~\ref{fig:firstpage}. We use the Lighthouse positioning system to measure the 3D position of the drone and payload. Notably, the payload has a mass of $m_p = 23\si{g}$ which makes more than $50\%$ of the quadrotors mass of $m_c = 41.1\si{g}$ and, thus, has a considerable effect on its dynamics. In experiments, we found that simply using the firmware position controllers was not sufficient and resulted in instability.

To obtain the dynamics model of the quadrotor-payload system, we model both rigid bodies as point masses connected by a massless inextensible cable, with an acceleration input to the drone $\ub \in \mathbb{R}^3$. We describe the generalized coordinates as $\bm{q} =[p_x \ \ p_y \ \ p_z \ \ \phi \ \ \theta]^T$ where $p$ denotes the position of the quadrotor and $\phi$ and $\theta$ denote the azimuth and polar angle of the pendulum. We obtain the simplified equations of motion via Lagrange’s method, which captures the dominant coupling between the vehicle and payload. We note that this is a rather simplified model, as it neither considers the nonlinearities in the quadrotor control nor any aerodynamic effects on the system but we found it to be sufficient for the proposed navigation task. Further, we found that there is a considerable effect of the generated airflow of the quadrotor (also known as \emph{downwash force}) on the pendulum dynamics. Since this force is difficult to model, we assume the pendulum dynamics include linear damping terms $\beta_{\phi} \dot{\phi}$ and $\beta_{\theta}\dot{\theta}$ with coefficients to be estimated online. We treat the pendulum length $L=0.52\si{m}$ as unknown and estimate it from a uniform prior $\mathcal{U}(0.3, 0.9)$.

The task is to track a square trajectory of length $1.5\,\si{m}$ while satisfying obstacle-avoidance constraints around three obstacles $\{\mathcal{O}_1,\mathcal{O}_2,\mathcal{O}_3\}$ as shown in Fig.~\ref{fig:firstpage} and Fig.~\ref{fig:Experiments}(a). Physically, the quadrotor-payload system is able to fly below obstacle $\mathcal{O}_1$ and above $\mathcal{O}_2$ but only does so when the estimated length is correct.
We report two different settings, one in which damping is assumed to be zero ($L$ only) and one in which we estimate $L, \beta_{\phi}$ and $\beta_{\theta}$ simultaneously.  
The controller runs at $30\,\si{Hz}$, and uses the same hyperparameters as in Sec.~\ref{sec:Baselines}.

\paragraph{Discussion} First, we note that PRMPPI ensures constraint satisfaction throughout all runs: no collisions are observed, and $h(\xb_t)$ remains nonnegative with small margins near tight passages, see Fig.~\ref{fig:Experiments}(d). Fallbacks to the robust trajectory are concentrated around $\mathcal{O}_3$, where the simplified dynamics and payload coupling make many cost-optimizing nominal rollouts infeasible. As online learning progresses in the three-parameter setting, the controller exploits tighter paths: After the first lap, the drone flies below $\mathcal{O}_1$ and after the second lap starts flying above $\mathcal{O}_2$ since the estimated length is below $60\si{cm}$. This results in improved tracking performance per lap as indicated by the RMSE in Fig.~\ref{fig:Experiments}(c). 
In contrast, when only estimating the length $L$ without damping, the swinging amplitude is lower than expected which results in overestimation of the length as seen in Fig.~\ref{fig:Experiments}(b). Consequently, the quadrotor is not able to take shortcuts which results in a worse RMSE. Therefore, we highlight that the performance of the proposed combined online parameter estimation and control method is naturally limited by the expressivity of the model and, thus, needs to be carefully designed. \rebuttaladd{Since our models omit several nonlinearities, we do not expect the learned parameters to match their true values exactly. In practice, however, the simplified dynamics are sufficient to keep the true parameters within the support of the belief, enabling robot safety despite model mismatch.}

\rebuttaladd{Since the true damping is unknown, we used a broad prior of \(\mathcal{U}(0, 0.1)\) for the damping coefficients. Despite this misspecification, SVGD is able to recover and converge to a distribution centered around \(0.15\,\mathrm{kg\cdot m^2/s}\).
}
Although the estimated length in the three-parameter setting underestimates the true length, it can be seen that the true length is in the support of the learned distribution.

\section{Conclusions and Future Work}
In this paper, we present parameter-robust MPPI, which combines an SVGD-based particle filter for online parameter learning with a controller that simultaneously optimizes nominal and robust trajectories. By enforcing a joint chance constraint over a receding horizon, our method enables safe online parameter learning and improves both performance and safety compared to existing baselines. The proposed parallel optimization structure proved effective in simulations and in hardware experiments on a quadrotor–payload system, demonstrating its practicality for real-world robotics. Future work will address the belief mismatch noted in Remark~\ref{remark1} and the coupling between parameter and state estimation.

\balance
\hypersetup{urlcolor=black}
\bibliographystyle{IEEEtran}
\vspace{-0.1cm}
\bibliography{references.bib}

@article{dawson2022learning,
  title={Learning safe, generalizable perception-based hybrid control with certificates},
  author={Dawson, Charles and Lowenkamp, Bethany and Goff, Dylan and Fan, Chuchu},
  journal={IEEE Robotics and Automation Letters},
  volume={7},
  number={2},
  pages={1904--1911},
  year={2022},
  publisher={IEEE}
}

@article{williams2017model,
  title={Model predictive path integral control: From theory to parallel computation},
  author={Williams, Grady and Aldrich, Andrew and Theodorou, Evangelos A},
  journal={Journal of Guidance, Control, and Dynamics},
  volume={40},
  number={2},
  pages={344--357},
  year={2017},
  publisher={American Institute of Aeronautics and Astronautics}
}

@article{liu2016stein,
  title={Stein variational gradient descent: A general purpose bayesian inference algorithm},
  author={Liu, Qiang and Wang, Dilin},
  journal={Advances in neural information processing systems},
  volume={29},
  year={2016}
}

@article{shafer2008tutorial,
  title={A tutorial on conformal prediction.},
  author={Shafer, Glenn and Vovk, Vladimir},
  journal={Journal of Machine Learning Research},
  volume={9},
  number={3},
  year={2008}
}

@book{thrun2005probabilistic,
	address = {Cambridge, Mass.},
	author = {Thrun, Sebastian and Burgard, Wolfram and Fox, Dieter},
	description = {Probabilistic Robotics Intelligent Robotics and Autonomous Agents: Amazon.de: Sebastian Thrun, Wolfram Burgard, Dieter Fox: Bücher},
	isbn = {0262201623 9780262201629},
	publisher = {MIT Press},
	refid = {58451645},
	title = {Probabilistic robotics},
	year = 2005
}

@book{silverman2018density,
  title={Density estimation for statistics and data analysis},
  author={Silverman, Bernard W},
  year={2018},
  publisher={Routledge}
}

@article{shi2023finite,
  title={A finite-particle convergence rate for stein variational gradient descent},
  author={Shi, Jiaxin and Mackey, Lester},
  journal={Advances in Neural Information Processing Systems},
  volume={36},
  pages={26831--26844},
  year={2023}
}

@article{yuan2022safe,
  title={Safe-control-gym: A unified benchmark suite for safe learning-based control and reinforcement learning in robotics},
  author={Yuan, Zhaocong and Hall, Adam W and Zhou, Siqi and Brunke, Lukas and Greeff, Melissa and Panerati, Jacopo and Schoellig, Angela P},
  journal={IEEE Robotics and Automation Letters},
  volume={7},
  number={4},
  pages={11142--11149},
  year={2022},
  publisher={IEEE}
}

@article{hewing2019cautious,
  title={Cautious model predictive control using gaussian process regression},
  author={Hewing, Lukas and Kabzan, Juraj and Zeilinger, Melanie N},
  journal={IEEE Transactions on Control Systems Technology},
  volume={28},
  number={6},
  pages={2736--2743},
  year={2019},
  publisher={IEEE}
}

@article{abraham2020model,
  title={Model-based generalization under parameter uncertainty using path integral control},
  author={Abraham, Ian and Handa, Ankur and Ratliff, Nathan and Lowrey, Kendall and Murphey, Todd D and Fox, Dieter},
  journal={IEEE Robotics and Automation Letters},
  volume={5},
  number={2},
  pages={2864--2871},
  year={2020},
  publisher={IEEE}
}

@INPROCEEDINGS{Barcelos-RSS-21, 
    AUTHOR    = {Lucas Barcelos AND Alexander Lambert AND Rafael Oliveira AND Paulo Borges AND Byron Boots AND Fabio Ramos}, 
    TITLE     = {{Dual Online Stein Variational Inference for Control and Dynamics}}, 
    BOOKTITLE = {Proceedings of Robotics: Science and Systems}, 
    YEAR      = {2021}, 
    ADDRESS   = {Virtual}, 
    MONTH     = {July}, 
    DOI       = {10.15607/RSS.2021.XVII.068} 
}

@book{ljung1983theory,
  title={Theory and practice of recursive identification},
  author={Ljung, Lennart and S{\"o}derstr{\"o}m, Torsten},
  year={1983},
  publisher={MIT press}
}

@inproceedings{barcelos2020disco,
  title={Disco: Double likelihood-free inference stochastic control},
  author={Barcelos, Lucas and Oliveira, Rafael and Possas, Rafael and Ott, Lionel and Ramos, Fabio},
  booktitle={International Conference on Robotics and Automation (ICRA)},
  pages={10969--10975},
  year={2020},
  organization={IEEE}
}

@INPROCEEDINGS{9636325,
  author={Ekal, Monica and Albee, Keenan and Coltin, Brian and Ventura, Rodrigo and Linares, Richard and Miller, David W.},
  booktitle={International Conference on Intelligent Robots and Systems (IROS)}, 
  title={Online Information-Aware Motion Planning with Inertial Parameter Learning for Robotic Free-Flyers}, 
  year={2021},
  volume={},
  number={},
  pages={8766-8773},
  doi={10.1109/IROS51168.2021.9636325}}

@inproceedings{rohr2023credible,
  title={Credible online dynamics learning for hybrid UAVs},
  author={Rohr, David and Lawrance, Nicholas and Andersson, Olov and Siegwart, Roland},
  booktitle={International Conference on Robotics and Automation (ICRA)},
  pages={1305--1311},
  year={2023},
  organization={IEEE}
}

@inproceedings{peng2018sim,
  title={Sim-to-real transfer of robotic control with dynamics randomization},
  author={Peng, Xue Bin and Andrychowicz, Marcin and Zaremba, Wojciech and Abbeel, Pieter},
  booktitle={international conference on robotics and automation (ICRA)},
  pages={3803--3810},
  year={2018},
  organization={IEEE}
}

@inproceedings{heiden2022probabilistic,
  title={Probabilistic inference of simulation parameters via parallel differentiable simulation},
  author={Heiden, Eric and Denniston, Christopher E and Millard, David and Ramos, Fabio and Sukhatme, Gaurav S},
  booktitle={International Conference on Robotics and Automation (ICRA)},
  pages={3638--3645},
  year={2022},
  organization={IEEE}
}

@article{deisenroth2013gaussian,
  title={Gaussian processes for data-efficient learning in robotics and control},
  author={Deisenroth, Marc Peter and Fox, Dieter and Rasmussen, Carl Edward},
  journal={IEEE transactions on pattern analysis and machine intelligence},
  volume={37},
  number={2},
  year={2013},
  publisher={IEEE}
}

@article{mayne2005robust,
  title={Robust model predictive control of constrained linear systems with bounded disturbances},
  author={Mayne, David Q and Seron, Mar{\'\i}a M and Rakovi{\'c}, Sa{\v{s}}a V},
  journal={Automatica},
  volume={41},
  number={2},
  pages={219--224},
  year={2005},
  publisher={Elsevier}
}

@article{mitchell2005time,
  title={A time-dependent Hamilton-Jacobi formulation of reachable sets for continuous dynamic games},
  author={Mitchell, Ian M and Bayen, Alexandre M and Tomlin, Claire J},
  journal={IEEE Transactions on automatic control},
  volume={50},
  year={2005},
  publisher={IEEE}
}

@article{knoedler2025safety,
  title={Safety on the Fly: Constructing Robust Safety Filters via Policy Control Barrier Functions at Runtime},
  author={Knoedler, Luzia and So, Oswin and Yin, Ji and Black, Mitchell and Serlin, Zachary and Tsiotras, Panagiotis and Alonso-Mora, Javier and Fan, Chuchu},
  journal={IEEE Robotics and Automation Letters},
  year={2025},
  publisher={IEEE}
}

@article{ostafew2016robust,
  title={Robust constrained learning-based NMPC enabling reliable mobile robot path tracking},
  author={Ostafew, Chris J and Schoellig, Angela P and Barfoot, Timothy D},
  journal={The International Journal of Robotics Research},
  volume={35},
  number={13},
  year={2016},
  publisher={SAGE Publications Sage UK: London, England}
}

@inproceedings{deisenroth2011pilco,
  title={PILCO: A model-based and data-efficient approach to policy search},
  author={Deisenroth, Marc and Rasmussen, Carl E},
  booktitle={Proceedings of the 28th International Conference on machine learning (ICML-11)},
  year={2011}
}

@article{hewing2020learning,
  title={Learning-based model predictive control: Toward safe learning in control},
  author={Hewing, Lukas and Wabersich, Kim P and Menner, Marcel and Zeilinger, Melanie N},
  journal={Annual Review of Control, Robotics, and Autonomous Systems},
  volume={3},
  year={2020},
  publisher={Annual Reviews}
}

@inproceedings{mehta2020active,
  title={Active domain randomization},
  author={Mehta, Bhairav and Diaz, Manfred and Golemo, Florian and Pal, Christopher J and Paull, Liam},
  booktitle={Conference on Robot Learning},
  year={2020},
  organization={PMLR}
}

@inproceedings{possas2020online,
  title={Online bayessim for combined simulator parameter inference and policy improvement},
  author={Possas, Rafael and Barcelos, Lucas and Oliveira, Rafael and Fox, Dieter and Ramos, Fabio},
  booktitle={International Conference on Intelligent Robots and Systems (IROS)},
  pages={5445--5452},
  year={2020},
  organization={IEEE}
}

@inproceedings{wang2021adaptive,
  title={Adaptive risk sensitive model predictive control with stochastic search},
  author={Wang, Ziyi and So, Oswin and Lee, Keuntaek and Theodorou, Evangelos A},
  booktitle={Learning for Dynamics and Control},
  pages={510--522},
  year={2021},
  organization={PMLR}
}

\end{document}